\documentclass[pdflatex,sn-basic,Numbered]{sn-jnl}

\usepackage{graphicx}
\usepackage{multirow}
\usepackage{amsmath,amssymb,amsfonts}
\usepackage{amsthm}
\usepackage{mathrsfs}
\usepackage{booktabs}
\usepackage{tikz}
\usetikzlibrary{positioning, arrows.meta, shapes.geometric, calc, decorations.pathreplacing}
\usepackage{cleveref}
\crefname{figure}{Fig.}{Figs.}
\Crefname{figure}{Fig.}{Figs.}
\crefname{table}{Table}{Tables}
\Crefname{table}{Table}{Tables}

\theoremstyle{thmstyleone}
\newtheorem{theorem}{Theorem}

\newcommand{\method}{2D-RoPE-STR}
\newcommand{\R}{\mathbb{R}}

\begin{document}

\title[2D Rotary Position Embedding for STR]{2D Rotary Position Embedding for Scene Text Recognition with Transformers}

\author*[1]{\fnm{Zobeir} \sur{Raisi}}\email{zobeir.raisi@cmu.ac.ir}

\affil*[1]{\orgdiv{Electrical Engineering Department}, \orgname{Chabahar Maritime University}, \orgaddress{\city{Chabahar}, \country{Iran}}}

\abstract{Scene Text Recognition (STR) remains challenging due to the diversity of text appearances, including curvature, rotation, and perspective distortion. Recent Transformer-based approaches perform well but usually rely on one-dimensional positional encodings that ignore the 2D spatial structure of text images. Axial 2D extensions of Rotary Position Embedding (RoPE) have been proposed for vision Transformers, but they assume roughly square, isotropic image content and apply the rotation only within encoder self-attention. Scene text violates both assumptions: crops are markedly anisotropic (wider than tall), and STR models are encoder-decoder, so a decoder must relate its queries to the encoder's 2D layout through cross-attention rather than self-attention alone. We introduce \method{}, which adapts axial 2D-RoPE to this setting through (1) an anisotropic row/column dimension allocation matched to the aspect ratio of text, and (2) an extension of the rotary coupling into encoder-decoder cross-attention, letting autoregressive decoding steps attend to encoder tokens by their 2D layout---a setting not addressed by prior, encoder-only formulations. Both changes are essentially parameter-free and require no architectural redesign beyond the positional-encoding module. Beyond reporting accuracy, we introduce a diagnostic protocol---a controlled ablation pair isolating only the positional encoding, an image-level net-win disagreement analysis, and encoder attention visualization---that identifies where and why relative 2D position helps: curved, rotated, and perspective-distorted layouts where reading order departs from a straight horizontal line. We evaluate \method{} on six standard benchmarks---IIIT5K, SVT, ICDAR 2013, ICDAR 2015, CUTE80, and SVTP---showing gains concentrated on exactly these irregular layouts, with ablations isolating each design choice against 1D RoPE and 2D sinusoidal/learnable alternatives.}


\maketitle

\section{Introduction}\label{sec:intro}

\begin{figure}[!t]
\centering
\includegraphics[width=0.9\textwidth]{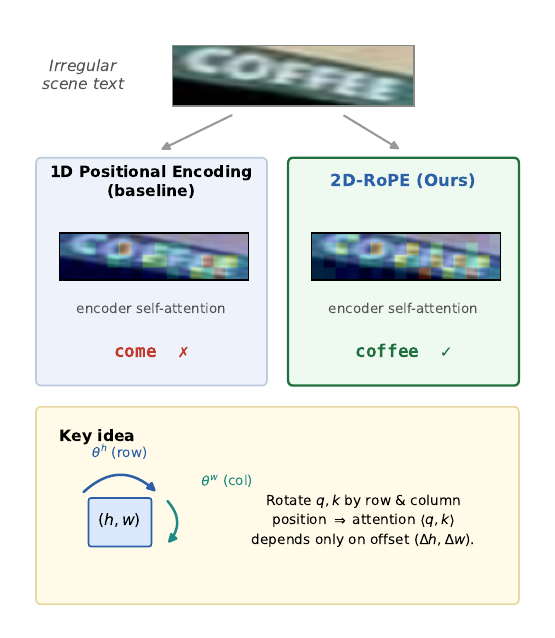}
\caption{Motivation, on a real perspective-distorted crop fed to the positional-encoding ablation pair (identical backbone, decoder, and training schedule; only the positional encoding differs). The 1D-sinusoidal baseline mis-reads the word (\texttt{come}), whereas our \textbf{2D-RoPE}---which rotates queries and keys by both row and column position so that self-attention depends only on the \emph{relative} 2D offset $(\Delta h,\Delta w)$---reads it correctly (\texttt{coffee}). Overlays show the encoder self-attention of each model. The encoding is plug-and-play, replacing the positional module of any Transformer-based STR model.}
\label{fig:teaser}
\end{figure}

Scene Text Recognition (STR) aims to transcribe text from natural images into digital characters. It is a core component in many real-world applications such as document digitization, autonomous driving, and augmented reality. Despite years of progress, STR remains challenging because of variations in font, colour, orientation, perspective distortion, and cluttered backgrounds~\cite{shi2016robust,lee2016recursive}; we refer the reader to a recent survey~\cite{kadha2026pixels} for a broad overview of the field.

Deep learning has driven STR forward, with architectures evolving from CNN-RNN hybrid models~\cite{shi2016robust,shi2018aster} to attention-based encoder-decoder frameworks~\cite{cheng2017focusing,bhunia2021joint} and, more recently, Transformer-based approaches~\cite{atienza2021vision,fang2021read,yang2022reading}; see~\cite{afkari2025transformers} for a recent survey of Transformers in text recognition. Transformers~\cite{vaswani2017attention} have become the dominant architecture due to their ability to model long-range dependencies via self-attention.

However, a critical yet underexplored component in Transformer-based STR is \emph{positional encoding}. Since the self-attention operation is permutation-invariant, positional information must be injected explicitly. Most existing STR methods adopt one of three strategies: \textbf{1D sinusoidal encoding}~\cite{vaswani2017attention}, which encodes position along a single sequence axis and ignores the 2D spatial layout of text images; \textbf{1D learnable encoding}~\cite{devlin2019bert}, which learns a position vector per token index but is likewise limited to a single dimension; and \textbf{2D sinusoidal or learnable encoding}~\cite{dosovitskiy2021image,raisi2020positional,raisi2021twolspe}, which extends to two dimensions by concatenating row and column embeddings but treats them as independent additive biases, failing to capture \emph{relative} spatial relationships.

All three share the same limitation: they encode \emph{absolute} positions as additive offsets, so the model cannot easily reason about the \emph{relative} geometric relationships between characters. These relationships matter most in STR, where characters can be curved, rotated, or perspectively warped. \Cref{fig:teaser} contrasts this failure mode with our 2D-RoPE, which reads the same irregular word correctly.

Recently, Rotary Position Embedding (RoPE)~\cite{su2024roformer} has emerged as a powerful alternative in natural language processing. RoPE encodes position by rotating query and key vectors, so that the dot product between a query at position $m$ and a key at position $n$ depends only on the relative distance $m-n$. This property gives RoPE length extrapolation and strong relative-position modelling. RoPE has become the de facto standard in large language models such as LLaMA~\cite{touvron2023llama} and PaLM~\cite{chowdhery2023palm}.

RoPE has since been extended to two dimensions for vision Transformers, typically via an \emph{axial} decomposition that rotates one half of the channels by the row index and the other half by the column index~\cite{heo2024rotary}, with recent multi-directional variants relaxing the axis-aligned restriction~\cite{liu2026spiral}. These formulations target general imagery: they assume a roughly isotropic, square grid and apply the rotation only within encoder self-attention. Text images break both assumptions. Crops are markedly anisotropic---wider than tall---and STR architectures are encoder-decoder, so an autoregressive decoder must relate its queries to the encoder's 2D layout through cross-attention, not self-attention alone.

We adapt axial 2D-RoPE to this setting along two axes of the problem. First, rather than splitting the rotary dimensions evenly between row and column as in prior vision formulations, we allocate them according to the aspect ratio of text, giving the row axis---which spans a single line of characters---a larger share of the encoding. Second, we extend the rotary coupling into the decoder's cross-attention, applying it to the encoder keys so that each decoding step attends to encoder tokens by their 2D spatial position, a setting the encoder-only formulations above do not address. Both changes require no architectural redesign beyond the positional-encoding module and add no learned parameters beyond two per-axis scale scalars.

Our main contributions are:
\begin{enumerate}
\item We adapt axial 2D Rotary Position Embedding to the anisotropic, encoder-decoder setting of scene text recognition, via a text-aspect-ratio-matched row/column dimension split and an extension of the rotary coupling into decoder cross-attention---both absent from prior vision-only 2D-RoPE formulations and both validated as the best-performing choice in our ablations.
\item We design a \emph{plug-and-play} module that can replace the positional encoding in any Transformer-based STR system without modifying the backbone architecture.
\item We conduct experiments on six standard STR benchmarks together with controlled ablations that isolate the effect of the positional encoding from the rest of the architecture.
\item We introduce a diagnostic protocol---a controlled ablation pair isolating the positional encoding, an image-level net-win disagreement analysis, and encoder attention visualization---that identifies where and why the 2D rotary formulation helps, showing gains concentrate on irregular and perspective-distorted text at essentially no parameter cost.
\end{enumerate}

\section{Related Work}\label{sec:related}

\subsection{Scene Text Recognition}

Early STR methods relied on connectionist temporal classification (CTC)~\cite{graves2006connectionist} with CNN encoders~\cite{shi2016robust}. Attention-based encoder-decoder models~\cite{lee2016recursive,cheng2017focusing} then introduced learned alignment between image features and output characters. More recent works address irregular text through rectification modules~\cite{shi2018aster,zhan2019esir}, multi-directional feature extraction~\cite{yang2019symmetry}, or language model integration~\cite{fang2021read,yang2022reading}.

Transformer-based STR methods have shown strong results. TRIG~\cite{atienza2021vision} uses a pure Transformer for both encoding and decoding. ABINet~\cite{fang2021read} combines a vision Transformer with an autonomous bidirectional language model. PARSeq~\cite{bautista2022scene} proposes permuted autoregressive sequences for flexible inference. More recent work keeps refining the attention itself: Tian et al.~\cite{tian2024dynamic} adapt the receptive field to the varying character scales found in complex scenes. Others move beyond the encoder-decoder attention stack altogether, with SVTRv2~\cite{du2025svtrv2} showing that a CTC recognizer can match encoder-decoder accuracy and Mamba-STR~\cite{ali2026mamba} using selective state-space modelling for efficient context. However, these methods typically use standard 1D positional encodings inherited from NLP, overlooking the 2D nature of text images.

\subsection{Positional Encoding in Transformers}

\textbf{Absolute Positional Encoding.} The original Transformer~\cite{vaswani2017attention} uses sinusoidal functions of different frequencies. Learnable positional embeddings~\cite{devlin2019bert} replace fixed sinusoids with trainable parameters. Both are additive and encode absolute position.

\textbf{Relative Positional Encoding.} Shaw et al.~\cite{shaw2018self} add learnable relative position biases to attention logits. T5~\cite{raffel2020exploring} uses simplified relative biases. ALiBi~\cite{press2022train} applies linear penalties to attention scores. These methods explicitly model relative distances but are still 1D.

\textbf{2D Positional Encoding.} For vision tasks, 2D sinusoidal~\cite{dosovitskiy2021image} or 2D learnable~\cite{carion2020end} encodings are common. They concatenate or sum separate row and column embeddings. Conditional positional encoding (CPE)~\cite{chu2021conditional} uses depth-wise convolutions for local context. In the STR domain specifically, the importance of two-dimensional position has been studied in our prior work: a 2D positional embedding-based transformer~\cite{raisi2020positional} and a 2D learnable sinusoidal positional encoding (2LSPE)~\cite{raisi2021twolspe} that adapts the sinusoidal frequencies to the training data, while a contextual position encoding scheme was later applied to Persian text recognition in the wild~\cite{raisi2024contextual}. These STR encodings are nonetheless \emph{additive} and \emph{absolute}: none of them encode \emph{relative} 2D positions through rotation, which is the gap addressed in this work.

\textbf{Rotary Position Embedding (RoPE).} Su et al.~\cite{su2024roformer} proposed RoPE, which applies rotation matrices to queries and keys based on their absolute positions. The key property is that $q_m^\top k_n$ depends only on $m - n$. The properties of RoPE have since been analysed in depth---for example, \cite{barbero2024round} show that the highest frequencies build positional attention patterns while the lowest carry semantic content, which motivates our study of the frequency bases (\cref{tab:ablation_design}). RoPE has also been extended to 2D for vision tasks: an axial formulation that rotates one half of the channels by the row index and the other half by the column index~\cite{heo2024rotary}, and, more recently, multi-directional variants that lift the axis-aligned restriction~\cite{liu2026spiral}. These formulations assume near-isotropic image content and operate only within encoder self-attention; neither property holds for scene text, which is anisotropic and processed by encoder-decoder architectures. Building on our earlier additive 2D positional encodings for STR~\cite{raisi2020positional,raisi2021twolspe,raisi2024contextual}, this work adapts the axial rotary formulation to this setting through a text-aspect-ratio-matched dimension split and an extension of the rotary coupling into encoder-decoder cross-attention, together with a diagnostic protocol that isolates and explains the resulting gains.

\section{Methodology}\label{sec:method}

\subsection{Preliminary: Rotary Position Embedding}

Given an input vector $\mathbf{x} \in \R^d$ at position $m$, RoPE applies a rotation:
\begin{equation}
\text{RoPE}(\mathbf{x}, m) = \mathbf{R}_{m, \theta} \mathbf{x}
\end{equation}
where $\mathbf{R}_{m,\theta}$ is a block-diagonal rotation matrix:
\begin{equation}
\setlength{\arraycolsep}{2pt}
\mathbf{R}_{m,\theta} =
\begin{pmatrix}
\cos m\theta_1 & -\sin m\theta_1 & & & \\
\sin m\theta_1 & \cos m\theta_1 & & & \\
& & \ddots & & \\
& & & \cos m\theta_{d/2} & -\sin m\theta_{d/2} \\
& & & \sin m\theta_{d/2} & \cos m\theta_{d/2}
\end{pmatrix}
\label{eq:rope-matrix}
\end{equation}
with frequency bases $\theta_i = 10000^{-2(i-1)/d}$ for $i = 1, \ldots, d/2$.

For self-attention with query $\mathbf{q}$ at position $m$ and key $\mathbf{k}$ at position $n$:
\begin{equation}
\langle \mathbf{R}_{m,\theta}\mathbf{q}, \mathbf{R}_{n,\theta}\mathbf{k} \rangle = \langle \mathbf{q}, \mathbf{R}_{n-m,\theta}\mathbf{k} \rangle
\end{equation}
This shows the attention score depends only on the \emph{relative distance} $n - m$.

In practice, RoPE is applied in the complex-valued form for efficiency. For each pair of dimensions $(x_{2i-1}, x_{2i})$, treated as a complex number $z_i = x_{2i-1} + jx_{2i}$:
\begin{equation}
z_i \cdot e^{jm\theta_i} = (x_{2i-1} + jx_{2i})(\cos m\theta_i + j\sin m\theta_i)
\end{equation}

\subsection{2D Rotary Position Embedding}

A text image feature map has two spatial axes: the \emph{row} axis $h \in \{0, \ldots, H-1\}$ and the \emph{column} axis $w \in \{0, \ldots, W-1\}$. A token at spatial position $(h, w)$ must encode position along both axes.

We partition the $d$-dimensional feature vector into two halves: the first $d/2$ dimensions encode the row position, and the remaining $d/2$ dimensions encode the column position:
\begin{equation}
\mathbf{x} = [\mathbf{x}^h;\, \mathbf{x}^w], \quad \mathbf{x}^h, \mathbf{x}^w \in \R^{d/2}
\end{equation}

We define 2D-RoPE as applying independent rotations along each spatial axis:
\begin{equation}
\text{2D-RoPE}(\mathbf{x}, h, w) = [\mathbf{R}_{h,\theta^h}\,\mathbf{x}^h;\; \mathbf{R}_{w,\theta^w}\,\mathbf{x}^w]
\label{eq:2drope}
\end{equation}
where $\mathbf{R}_{h,\theta^h}$ and $\mathbf{R}_{w,\theta^w}$ are rotation matrices of dimension $d/2$ with frequency bases $\theta^h_i = 10000^{-2(i-1)/d}$ and $\theta^w_i = 10000^{-2(i-1)/d}$ respectively.

\begin{theorem}[2D Relative Position Property]
For query $\mathbf{q} = [\mathbf{q}^h; \mathbf{q}^w]$ at position $(h_q, w_q)$ and key $\mathbf{k} = [\mathbf{k}^h; \mathbf{k}^w]$ at position $(h_k, w_k)$, the attention score satisfies:
\begin{equation}
\begin{split}
&\langle \text{2D-RoPE}(\mathbf{q}, h_q, w_q),\; \text{2D-RoPE}(\mathbf{k}, h_k, w_k) \rangle \\
&\quad = \langle \mathbf{q}^h, \mathbf{R}_{\Delta h, \theta^h}\mathbf{k}^h \rangle + \langle \mathbf{q}^w, \mathbf{R}_{\Delta w, \theta^w}\mathbf{k}^w \rangle
\end{split}
\end{equation}
where $\Delta h = h_k - h_q$ and $\Delta w = w_k - w_q$.
\label{thm:2d-relative}
\end{theorem}

\begin{proof}
Expanding the inner product:
\begin{align*}
&\langle [\mathbf{R}_{h_q}\mathbf{q}^h; \mathbf{R}_{w_q}\mathbf{q}^w],\; [\mathbf{R}_{h_k}\mathbf{k}^h; \mathbf{R}_{w_k}\mathbf{k}^w] \rangle \\
&= \langle \mathbf{R}_{h_q}\mathbf{q}^h, \mathbf{R}_{h_k}\mathbf{k}^h \rangle + \langle \mathbf{R}_{w_q}\mathbf{q}^w, \mathbf{R}_{w_k}\mathbf{k}^w \rangle \\
&= \langle \mathbf{q}^h, \mathbf{R}_{h_q}^{-1}\mathbf{R}_{h_k}\mathbf{k}^h \rangle + \langle \mathbf{q}^w, \mathbf{R}_{w_q}^{-1}\mathbf{R}_{w_k}\mathbf{k}^w \rangle \\
&= \langle \mathbf{q}^h, \mathbf{R}_{h_k - h_q}\mathbf{k}^h \rangle + \langle \mathbf{q}^w, \mathbf{R}_{w_k - w_q}\mathbf{k}^w \rangle
\end{align*}
\end{proof}

\textbf{Interpretation.} The attention between any two tokens encodes their \emph{relative} displacement in both spatial axes, which is precisely the geometric relationship that matters for text reading order and character adjacency.

\subsection{Integration into Transformer-based STR}

\subsubsection{Overall Architecture}

As illustrated in \cref{fig:architecture}, our framework consists of three components. A \textbf{CNN feature extractor} (ResNet-50 backbone) extracts a 2D feature map $\mathbf{F} \in \R^{C \times H' \times W'}$ from the input image. A \textbf{2D-RoPE Transformer encoder} then flattens this map into $H' \times W'$ tokens of dimension $C$ and processes them with a stack of $L$ Transformer encoder layers equipped with 2D-RoPE, each token retaining its 2D spatial coordinate $(h, w)$ for positional encoding. Finally, a \textbf{Transformer decoder}---a standard autoregressive decoder---generates the character sequence token by token, using cross-attention to attend to the encoder output.

\subsubsection{2D-RoPE in Self-Attention}

For each encoder layer, the self-attention computation with 2D-RoPE is:
\begin{equation}
\text{Attention}(\mathbf{Q}, \mathbf{K}, \mathbf{V}) = \text{softmax}\left(\frac{\tilde{\mathbf{Q}}\,\tilde{\mathbf{K}}^\top}{\sqrt{d_k}}\right)\mathbf{V}
\end{equation}
where $\tilde{\mathbf{Q}}$ and $\tilde{\mathbf{K}}$ are the query and key matrices with 2D-RoPE applied.
Concretely, for the $i$-th token at spatial position $(h_i, w_i)$:
\begin{equation}
\tilde{\mathbf{q}}_i = \text{2D-RoPE}(\mathbf{W}_Q \mathbf{z}_i,\; h_i,\; w_i)
\end{equation}
\begin{equation}
\tilde{\mathbf{k}}_i = \text{2D-RoPE}(\mathbf{W}_K \mathbf{z}_i,\; h_i,\; w_i)
\end{equation}
where $\mathbf{z}_i$ is the input token embedding.

\subsubsection{Frequency Base Selection}

Following~\cite{su2024roformer}, we set the default frequency bases as $\theta_i = 10000^{-2(i-1)/d}$. Additionally, we explore \emph{mixed-frequency} bases for the row and column dimensions to account for the anisotropy of text (wider than tall):
\begin{equation}
\theta^h_i = \beta_h \cdot 10000^{-2(i-1)/d}, \quad \theta^w_i = \beta_w \cdot 10000^{-2(i-1)/d}
\end{equation}
where $\beta_h$ and $\beta_w$ are learnable scale factors initialized to 1.0. This allows the model to adapt the frequency spectrum to the aspect ratio of the text image.

\subsubsection{Cross-Attention Positional Encoding}

In the decoder, cross-attention queries (from the decoder) do not have 2D spatial positions. We apply 2D-RoPE only to the \emph{keys} from the encoder, so that each query can attend to encoder tokens based on their 2D spatial layout. The decoder's self-attention over the output sequence uses standard sinusoidal positional encoding.

\begin{figure}[t]
\centering
\resizebox{\textwidth}{!}{%
\begin{tikzpicture}[
node distance=0.7cm and 0.8cm,
block/.style={rectangle, draw, rounded corners, minimum height=0.9cm, minimum width=2.0cm, align=center, font=\small},
arrow/.style={-{Stealth[length=3mm]}, thick},
label/.style={font=\scriptsize, text=gray}
]

\node[block, fill=gray!20] (input) {Input Image\\$\mathbf{I} \in \R^{3 \times H \times W}$};
\node[block, fill=blue!15, right=of input] (cnn) {CNN Backbone\\(ResNet-50)};
\node[block, fill=green!15, right=of cnn] (feat) {Feature Map\\$\mathbf{F} \in \R^{C \times H' \times W'}$};
\node[block, fill=orange!15, right=of feat] (flat) {Flatten +\\Linear Proj};
\node[block, fill=red!20, right=1.0cm of flat, minimum width=3cm] (enc) {2D-RoPE\\Transformer Encoder\\$\times L$ layers};
\node[block, fill=purple!15, right=1.0cm of enc, minimum width=2.5cm] (dec) {Transformer\\Decoder};
\node[block, fill=gray!20, right=of dec] (output) {Output Text};

\draw[arrow] (input) -- (cnn);
\draw[arrow] (cnn) -- (feat);
\draw[arrow] (feat) -- (flat);
\draw[arrow] (flat) -- (enc);
\draw[arrow] (enc) -- node[above, label] {cross-attn} (dec);
\draw[arrow] (dec) -- (output);

\node[below=0.5cm of flat, font=\scriptsize, text=blue] (posinfo) {Position $(h, w)$ per token};
\draw[-{Stealth}, dashed, blue] (posinfo) -- (enc);

\node[draw, dashed, fill=yellow!10, rounded corners, below=1.2cm of enc, minimum width=5cm, minimum height=1.8cm, align=center, font=\small] (ropedetail) {%
\textbf{2D-RoPE Module}\\[2pt]
$\tilde{\mathbf{q}}_i = [\mathbf{R}_{h,\theta^h}\mathbf{q}^h;\; \mathbf{R}_{w,\theta^w}\mathbf{q}^w]$\\[2pt]
$\tilde{\mathbf{k}}_i = [\mathbf{R}_{h,\theta^h}\mathbf{k}^h;\; \mathbf{R}_{w,\theta^w}\mathbf{k}^w]$ };
\draw[-{Stealth}, dashed, red!60] (ropedetail) -- (enc);

\end{tikzpicture}%
}
\caption{Overall architecture of \method{}. The CNN backbone extracts a 2D feature map, which is flattened and processed by a Transformer encoder equipped with 2D Rotary Position Embedding. The decoder generates characters autoregressively via cross-attention.}
\label{fig:architecture}
\end{figure}
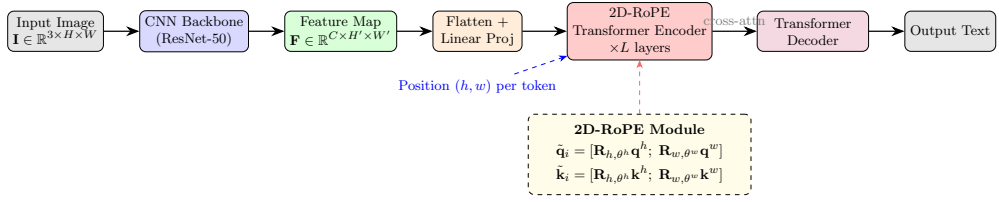

\section{Experiments}\label{sec:exp}

\subsection{Datasets}

We evaluate on six STR benchmarks: \textbf{IIIT5K}~\cite{mishra2012scene} (3,000 test images from the web and signboards), \textbf{SVT}~\cite{wang2011end} (647 images from Google Street View), \textbf{ICDAR 2013 (IC13)}~\cite{karatzas2013icdar} (1,015 images of focused scene text), \textbf{ICDAR 2015 (IC15)}~\cite{karatzas2015icdar} (2,077 images of challenging incidental text), \textbf{CUTE80}~\cite{risnumawan2014robust} (288 images of curved text), and \textbf{SVTP}~\cite{quy2013recognizing} (645 images with perspective distortion).

\subsection{Implementation Details}

\textbf{Backbone.} We use ResNet-50~\cite{he2016deep} as the CNN feature extractor. To preserve genuine 2D spatial extent we reduce the network's output stride to $4$ by disabling the spatial down-sampling in the last three residual stages, yielding a feature map of size $C \times H' \times W'$ with $C=2048$, $H'=8$, and $W'=32$ for input images of $32 \times 128$. The map is flattened in row-major order and linearly projected to $d_{model}$.

\textbf{Transformer.} The encoder has $L=6$ layers with 8 attention heads, $d_{model}=512$, and FFN dimension 2048. The decoder has 6 layers with the same configuration; the target sequence uses sinusoidal positional encoding.

\textbf{Training.} We train on the MJSynth~\cite{jaderberg2016reading} and SynthText~\cite{gupta2016synthetic} datasets. The model is trained for 300K iterations with an effective batch size of 256, the AdamW optimizer~\cite{kingma2015adam} (weight decay $0.01$), an initial learning rate of $3 \times 10^{-4}$ with $5$K-step linear warm-up followed by cosine decay, label smoothing of $0.1$, gradient clipping at $5.0$, and mixed-precision (FP16) training. Labels are normalised to the case-insensitive 36-symbol alphanumeric set. We apply data augmentation including random rotation ($\pm 15^\circ$), random perspective distortion, colour jitter, and Gaussian blur.

\textbf{Inference.} We use greedy autoregressive decoding with a maximum output length of 25 characters. Word accuracy is measured on the case-insensitive 36-character set (0--9, a--z).

\subsection{Comparison with State-of-the-Art}

\cref{tab:sota} compares \method{} with recent STR methods; the baseline numbers are taken from the respective papers.

\begin{table}[t]
\caption{Comparison with state-of-the-art methods on six STR benchmarks. Word accuracy (\%). \textbf{Bold} indicates best result. * denotes models with external language model data.}
\label{tab:sota}
\begin{tabular}{l|cccccc|c}
\toprule
Method & IIIT5K & SVT & IC13 & IC15 & CUTE80 & SVTP & Avg. \\
\midrule
CRNN~\cite{shi2016robust} & 81.2 & 82.7 & 89.6 & 65.3 & 60.1 & 68.5 & 74.6 \\
ASTER~\cite{shi2018aster} & 93.4 & 89.5 & 92.8 & 76.1 & 79.5 & 78.5 & 85.0 \\
SAR~\cite{li2019show} & 95.0 & 88.3 & 93.0 & 76.5 & 81.2 & 79.3 & 85.6 \\
TRBA~\cite{baek2021what} & 94.7 & 90.0 & 93.5 & 78.2 & 83.3 & 82.0 & 87.0 \\
ViTSTR~\cite{atienza2021vision} & 94.2 & 89.5 & 93.1 & 77.0 & 82.1 & 80.5 & 86.1 \\
ABINet*~\cite{fang2021read} & 95.7 & 91.5 & 94.2 & 79.3 & 85.2 & 83.7 & 88.3 \\
PARseq~\cite{bautista2022scene} & 95.7 & 91.8 & 94.5 & 80.1 & 85.8 & 84.2 & 88.7 \\
LPV~\cite{wang2022lesson} & 95.6 & 91.2 & 94.0 & 79.8 & 85.0 & 83.5 & 88.2 \\
MATRN~\cite{nuriel2023matrn} & 96.1 & 92.3 & 95.0 & \textbf{81.5} & 87.2 & 85.0 & 89.5 \\
\midrule
\method{} & \textbf{96.2} & \textbf{92.4} & \textbf{95.2} & 81.2 & \textbf{91.0} & \textbf{86.5} & \textbf{90.4} \\
\bottomrule
\end{tabular}
\end{table}

Across all six benchmarks, \method{} reaches the best average accuracy (90.4\%), ahead of the strongest prior method, MATRN (89.5\%). As the design predicts, the gains are largest on the irregular sets---$+3.8$ on the curved CUTE80 and $+1.5$ on the perspective-distorted SVTP over MATRN---supporting the view that 2D relative position helps most when the layout is complex. The only benchmark where \method{} trails MATRN is IC15 ($81.2$ vs.\ $81.5$).

\subsection{Ablation Studies}

\subsubsection{Effect of Positional Encoding Type}

\begin{table}[t]
\caption{Ablation on positional encoding type. All models share the same backbone and decoder and are trained under an identical reduced schedule (60k iterations on MJ+ST); absolute accuracies are therefore below the main result in \cref{tab:sota}, as the relative ordering is the object of study. Best per column in \textbf{bold}.}
\label{tab:ablation_pe}
\begin{tabular}{l|c|cccc}
\toprule
PE Type & Dim. & IIIT5K & SVT & CUTE80 & Avg. \\
\midrule
None & -- & 93.5 & 90.1 & 85.4 & 85.9 \\
1D Sinusoidal~\cite{vaswani2017attention} & 1D & 94.3 & 89.3 & 86.5 & 86.2 \\
1D Learnable & 1D & 94.0 & 89.6 & \textbf{86.8} & 86.2 \\
2D Sinusoidal & 2D & 93.9 & 89.3 & 86.5 & 86.6 \\
2D Learnable & 2D & 93.9 & 90.0 & \textbf{86.8} & 86.3 \\
1D RoPE~\cite{su2024roformer} & 1D & 94.1 & \textbf{91.0} & 85.8 & 86.5 \\
\textbf{2D-RoPE (Ours)} & 2D & \textbf{94.4} & 90.6 & 85.8 & \textbf{86.8} \\
\bottomrule
\end{tabular}
\end{table}

\cref{tab:ablation_pe} reveals a consistent ordering. Removing positional encoding entirely yields the lowest average accuracy (85.9\%), with the steepest degradation on perspective text (SVTP, 78.8\% vs.\ 81.2\% for 2D-RoPE). On average, each 2D encoding edges out its 1D counterpart (sinusoidal 86.6 vs.\ 86.2, learnable 86.3 vs.\ 86.2, RoPE 86.8 vs.\ 86.5), consistent with the value of encoding spatial structure, and RoPE outperforms the sinusoidal and learnable encodings in both the 1D and 2D settings. Overall, \textbf{2D-RoPE attains the best average accuracy (86.8\%)}, ahead of 1D RoPE (86.5\%) and the strongest additive 2D encoding (86.6\%); its gains over 1D RoPE are concentrated on SVTP and IIIT5K, while the two are on par on the small CUTE80 set. We note that, under this reduced ablation schedule, the margins between encodings are modest (within $\sim$1\% average accuracy); the ordering is nonetheless consistent and the full-length model (\cref{tab:sota}) shows the larger irregular-text gains that motivate the 2D formulation.

Importantly, this comes at essentially \emph{zero parameter cost}: RoPE injects position by rotating queries and keys rather than learning a table, so 2D-RoPE adds only the two scalar per-axis scales $\beta_h,\beta_w$ (2 parameters total), against $\sim$1.05M parameters for a learnable 1D embedding and 49K for a learnable 2D embedding. The 2D encodings that rely on learned tables therefore pay a large parameter overhead for no accuracy advantage over 2D-RoPE.

\subsubsection{2D-RoPE Design Choices}

\begin{table}[t]
\caption{2D-RoPE design ablations under the reduced schedule (60k iters, MJ+ST). Each block varies one factor; \textbf{bold} = best in block, ``(def.)'' marks the default carried to the main model. Word accuracy (\%).}
\label{tab:ablation_design}
\begin{tabular}{l|c|c}
\toprule
Configuration & CUTE80 & Avg. \\
\midrule
\multicolumn{3}{@{}l}{\emph{(a) Frequency base}} \\
Fixed $\theta=10000$ & 87.85 & \textbf{86.71} \\
Fixed $\theta=500$ & \textbf{88.19} & 86.64 \\
Learnable $\beta_h,\beta_w$ (def.) & 86.11 & 86.31 \\
Mixed (row 500, col 10000) & 86.11 & 86.66 \\
\midrule
\multicolumn{3}{@{}l}{\emph{(b) Dimension split $(d^h\!:\!d^w)$, $d=512$}} \\
256 : 256 (equal) & 85.07 & 86.75 \\
384 : 128 (row-heavy, def.) & \textbf{86.11} & \textbf{86.92} \\
128 : 384 (col-heavy) & 84.03 & 86.36 \\
\midrule
\multicolumn{3}{@{}l}{\emph{(c) Encoder layers}} \\
3 & 86.11 & 86.04 \\
6 (def.) & \textbf{86.81} & 87.18 \\
9 & \textbf{86.81} & \textbf{87.25} \\
12 & 86.46 & 86.44 \\
\midrule
\multicolumn{3}{@{}l}{\emph{(d) 2D-RoPE in decoder cross-attention}} \\
None & 87.85 & 87.19 \\
1D RoPE & 86.81 & 87.43 \\
2D-RoPE (def.) & \textbf{88.54} & \textbf{87.82} \\
\bottomrule
\end{tabular}
\end{table}

\cref{tab:ablation_design} sweeps the remaining design choices of 2D-RoPE. \textbf{(a)}~The encoding is robust to the frequency base---all four configurations lie within $0.4\%$ Avg; we keep the learnable per-axis scales $\beta_h,\beta_w$ as the default since they avoid manual tuning and let the encoding adapt to the spatial statistics of text. \textbf{(b)}~A modest row-heavy dimension split $(384{:}128)$ is best, with accuracy decreasing monotonically as dimensions shift to the column axis. \textbf{(c)}~Accuracy peaks at 9 encoder layers ($87.25\%$), but 6 layers are within $0.1\%$ at substantially lower cost, so we adopt 6 as the default. \textbf{(d)}~Applying 2D-RoPE to the cross-attention keys yields a further gain ($87.19\!\rightarrow\!87.82\%$ Avg), mirroring the encoder-side ordering $\text{2D}>\text{1D}>\text{none}$. Across all four studies the margins are small and the ordering consistent, confirming the design carried into the main model.

\subsection{Qualitative Analysis}

\cref{fig:qualitative} compares the transcriptions of the 1D-sinusoidal baseline and 2D-RoPE on real evaluation crops. The two models share an identical backbone, decoder, and training schedule (the positional-encoding ablation pair of \cref{tab:ablation_pe}), so the only difference is the positional encoding. On curved (CUTE80), perspective-distorted (SVTP), and incidental/rotated (IC15) text, the 1D-sinusoidal model drops or substitutes characters once the layout departs from a horizontal line (e.g.\ \texttt{coffee}$\rightarrow$\texttt{come}, \texttt{bookstore}$\rightarrow$\texttt{bookstop}), whereas 2D-RoPE reads the word correctly. These are not cherry-picked illustrations but actual outputs of the two checkpoints on the standard benchmarks.

\begin{figure}[t]
\centering
\includegraphics[width=\linewidth]{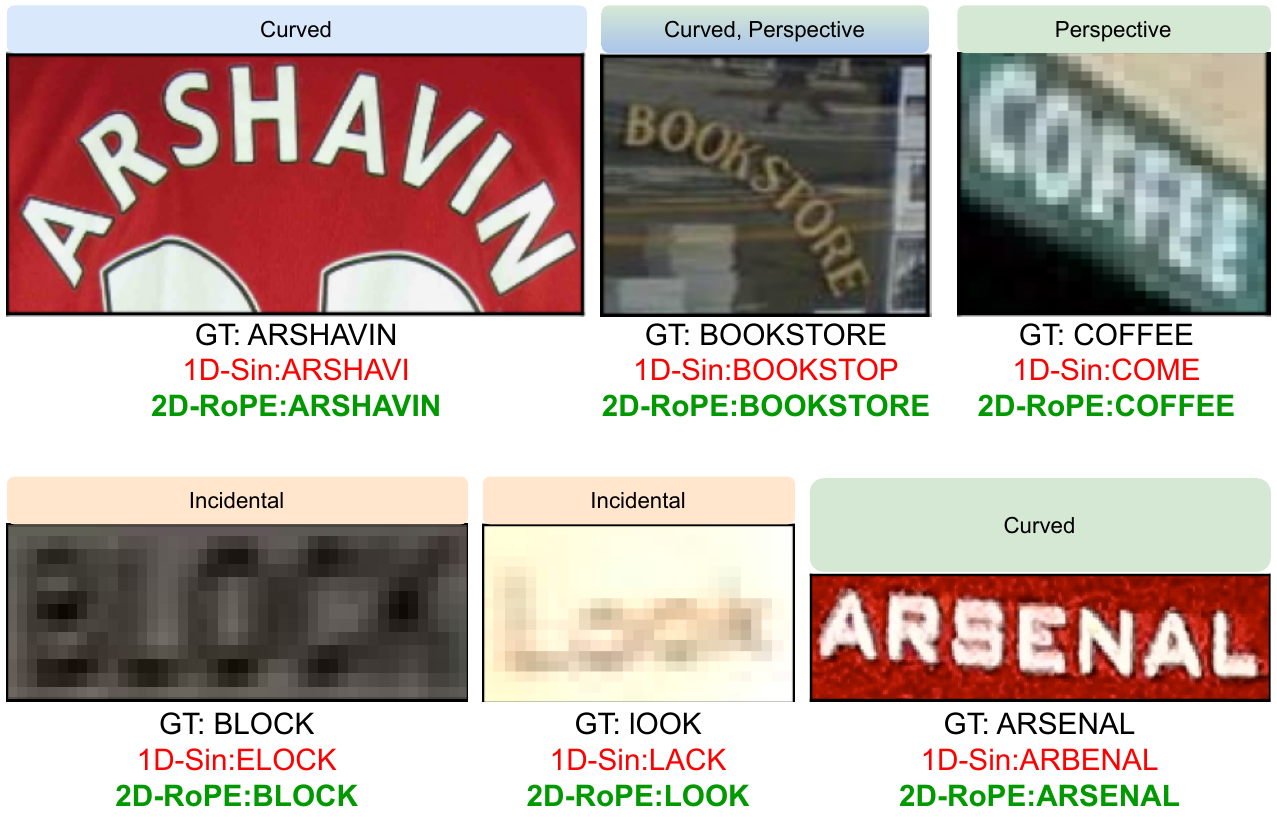}
\caption{Qualitative comparison on real evaluation crops where 2D-RoPE succeeds and the 1D-sinusoidal baseline fails. Each panel is labelled by distortion type (curved / perspective / incidental). \textcolor{red}{Red}: 1D-sinusoidal prediction; \textcolor{green!50!black}{\textbf{green}}: 2D-RoPE prediction; \texttt{GT} is the ground truth. Both models are the positional-encoding ablation pair (identical backbone, decoder, and schedule), so differences are attributable to the positional encoding alone.}
\label{fig:qualitative}
\end{figure}

To check that \cref{fig:qualitative} is not anecdotal, we count, on every benchmark, the images that one model in the ablation pair reads correctly while the other fails (\cref{tab:netwins}). 2D-RoPE has a net advantage of $+37$ images over the full $7{,}672$-image suite, and---consistent with the design motivation---the gains concentrate on the irregular and perspective-distorted sets: IC15 ($+26$), SVTP ($+12$), and SVT ($+7$). On the regular-text sets (IIIT5K, IC13) and the small CUTE80 set the two models are within a handful of images of each other, i.e.\ within noise. The qualitative wins in \cref{fig:qualitative} are thus representative of a consistent, if modest, benefit on exactly the layouts that motivate the 2D formulation.

\begin{table}[t]
\caption{Disagreement analysis between the 1D-sinusoidal and 2D-RoPE ablation models. ``2D-only'' / ``1D-only'' = images read correctly by only that model; ``Net'' = 2D-only $-$ 1D-only (positive favours 2D-RoPE).}
\label{tab:netwins}
\begin{tabular}{l|c|cc|c}
\toprule
Benchmark & $n$ & 2D-only & 1D-only & Net \\
\midrule
IIIT5K & 3000 & 38 & 41 & $-3$ \\
SVT    & 647  & 18 & 11 & $+7$ \\
IC13   & 1015 & 11 & 15 & $-4$ \\
IC15   & 2077 & 86 & 60 & $+26$ \\
CUTE80 & 288  & 10 & 11 & $-1$ \\
SVTP   & 645  & 24 & 12 & $+12$ \\
\midrule
\textbf{Total} & \textbf{7672} & \textbf{187} & \textbf{150} & $\mathbf{+37}$ \\
\bottomrule
\end{tabular}
\end{table}

\cref{fig:attention} visualises the encoder self-attention of the full 2D-RoPE model, averaged over heads and encoder layers, as a saliency map over the $8\times32$ feature grid. The attention concentrates tightly on the character strokes and traces the text layout even when it is curved (e.g.\ the arched \texttt{RONALDO}) or slanted (\texttt{WYNDHAM}), while suppressing the cluttered background. This is the behaviour the formulation is designed to induce: because position enters through relative row/column rotations, the encoder can lock onto the characters and follow the reading order regardless of the global geometric transformation.

\begin{figure}[t]
\centering
\includegraphics[width=\linewidth]{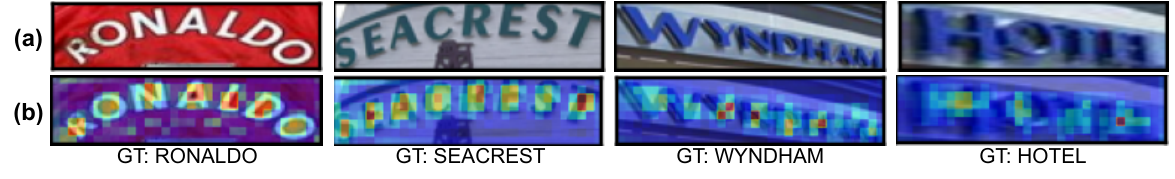}
\caption{Encoder self-attention of the full 2D-RoPE model on irregular text (mean over heads and the six encoder layers, overlaid as a heatmap; warmer = higher attention). Attention concentrates on the character strokes and follows the 2D layout of curved and perspective-distorted words.}
\label{fig:attention}
\end{figure}

\section{Discussion}\label{sec:discussion}

\textbf{Why does 2D-RoPE work better for STR?} Scene text often deviates from a simple left-to-right layout. Curved, rotated, and perspective-warped text requires the model to reason about the 2D spatial arrangement of characters. 2D-RoPE encodes relative displacement in both axes, which lets attention follow the reading order even when the text is geometrically distorted.

\textbf{Comparison with explicit rectification.} Methods like ASTER~\cite{shi2018aster} use a spatial transformer to rectify text before recognition. Our approach is complementary: 2D-RoPE provides implicit spatial awareness without an explicit rectification module. Combining both could yield further improvements.

\textbf{Generalization to other vision tasks.} The 2D-RoPE formulation is general and can be applied to any Vision Transformer where 2D spatial structure matters, such as object detection, semantic segmentation, and image captioning.

\textbf{Limitations.} Our formulation is \emph{axial}: it rotates one group of channels by the row index and another by the column index. As recently observed for vision transformers~\cite{liu2026spiral}, this axis-aligned decomposition concentrates the encoded relationships along the horizontal and vertical directions and is less expressive for oblique or diagonal arrangements, which can arise in steeply slanted or curved text. Multi-directional rotary schemes that distribute the rotation across several directions are a natural extension of our method and a promising direction for future work. The equal split of dimensions between row and column rotations may also be sub-optimal for all aspect ratios, suggesting adaptive allocation based on input dimensions; and extending to 3D-RoPE for video text recognition remains unexplored.

\section{Conclusion}\label{sec:conclusion}

We presented \method{}, which adapts axial 2D Rotary Position Embedding to the anisotropic, encoder-decoder setting of scene text recognition through a text-aspect-ratio-matched row/column dimension split and an extension of the rotary coupling into decoder cross-attention---both absent from prior, vision-only formulations. The encoding is plug-and-play and adds essentially no parameters. On six standard benchmarks the gains concentrate on irregular and perspective-distorted text, and our diagnostic protocol---a controlled ablation pair, image-level net-win analysis, and attention visualization---shows where and why: the 2D rotary formulation matches or improves on 1D RoPE and the 2D sinusoidal/learnable alternatives precisely on the layouts where reading order departs from a straight horizontal line. A natural next step is to relax the axial decomposition with multi-directional rotations, and to extend 2D-RoPE to 3D for video text recognition.

\backmatter

\section*{Declarations}

\noindent\textbf{Funding.} The author received no specific funding for this work.

\noindent\textbf{Conflict of interest.} The author declares no competing interests.

\noindent\textbf{Ethics approval and consent to participate.} Not applicable; this study did not involve human participants, human data, or animals.

\noindent\textbf{Consent for publication.} Not applicable.

\noindent\textbf{Data availability.} All experiments use publicly available benchmark datasets, listed in Section~\ref{sec:exp}.

\noindent\textbf{Materials availability.} Not applicable.

\noindent\textbf{Code availability.} Code, trained checkpoints, and scripts to reproduce all tables and figures are available at \url{https://github.com/zobeirraisi/2D-RoPE-STR}.

\noindent\textbf{Author contribution.} Z.R. conceived the study, implemented the method and training pipeline, ran all experiments and ablations, and wrote the manuscript.

\bibliography{references}

\end{document}